\documentclass[12pt, anon]{l4dc2024}

\title[Neural Operators for Boundary Stabilization of Stop-and-go Traffic]{Neural Operators for Boundary Stabilization of Stop-and-go Traffic}
\usepackage{times}
\usepackage{natbib}
\usepackage{amssymb}
\usepackage{makecell}
\usepackage{caption,subcaption}
\usepackage{graphicx}
\usepackage{amsmath,amsfonts}
\usepackage{xcolor}
\usepackage{physics}

\author{%
 \Name{Yihuai Zhang} \Email{yzhang169@connect.hkust-gz.edu.cn}\\
 \Name{Ruiguo Zhong} \Email{rzhong151@connect.hkust-gz.edu.cn}\\
 \addr Thrust of Intelligent Transportation, The Hong Kong University of Science and Technology (Guangzhou)
 \AND
\Name{Huan Yu}\thanks{corresponding author}\Email{huanyu@ust.hk} \\
 \addr Thrust of Intelligent Transportation, The Hong Kong University of Science and Technology (Guangzhou)\\
 Department of Civil and Environmental Engineering, The Hong Kong University of Science and Technology%
}

\begin{document}
\allowdisplaybreaks
\maketitle

\begin{abstract}%
This paper introduces a novel approach to PDE boundary control design using neural operators to alleviate stop-and-go instabilities in congested traffic flow. Our framework leverages neural operators to design control strategies for traffic flow systems. The traffic dynamics are described by the Aw-Rascle-Zhang (ARZ) model, which comprises a set of second-order coupled hyperbolic partial differential equations (PDEs). Backstepping method is widely used for boundary control of such PDE systems. The PDE model-based control design can be time-consuming and require intensive depth of expertise since it involves 
constructing and solving backstepping control kernel. To overcome these challenges, we present two distinct neural operator (NO) learning schemes aimed at stabilizing the traffic PDE system. The first scheme embeds NO-approximated gain kernels within a predefined backstepping controller, while the second one directly learns a boundary control law. The Lyapunov analysis is conducted to evaluate the stability of the NO-approximated gain kernels and control law. It is proved that the NO-based closed-loop system is practical stable under certain approximation accuracy conditions in NO-learning. To validate the efficacy of the proposed approach, simulations are conducted to compare the performance of the two neural operator controllers with a PDE backstepping controller and a Proportional Integral (PI) controller. While the NO-approximated methods exhibit higher errors compared to the backstepping controller, they consistently outperform the PI controller, demonstrating faster computation speeds across all scenarios. This result suggests that neural operators can significantly expedite and simplify the process of obtaining boundary controllers in traffic PDE systems.

\end{abstract}


\section{Introduction}
Stop-and-go traffic oscillations are a common phenomenon on freeways, causing increased waiting time, fuel consumption and traffic accidents\cite{belletti2015prediction,de2011traffic,schonhof2007empirical}. Therefore, it is of interest to eliminate the traffic congestion with boundary control, implemented with ramp metering or varying speed limits. To describe the dynamics of the traffic oscillations,  \cite{aw2000resurrection,zhang2002non} introduced the Aw-Rascle-Zhang (ARZ) model consisting of a 2 $\times$ 2 hyperbolic partial differential equations(PDEs) to describe the density and velocity of traffic systems. Backstepping method is well-studied for stabilizing PDEs \cite{krstic2008boundary, vazquez2011backstepping,anfinsen2019adaptive}. Considering disturbances and delays of hyperbolic PDEs, \cite{auriol2020robust} adopted a robust output feedback controller based on backstepping to guarantee the robustness of hyperbolic PDEs. With applications to traffic system, \cite{yu2022traffic} firstly applied backstepping control method to traffic system based on ARZ PDE model. 

The implementation of the aforementioned control designs involves numerical schemes, rendering the process time-consuming, particularly when addressing the solution of PDEs and constructing backstepping control kernels. Moreover, it necessitates substantial expertise, given the intricacy involved in constructing the functional mapping of the backstepping transformation. Recently, machine learning (ML) methods have emerged as powerful tools for solving complex algebraic and PDEs. 
Among these, neural operator(NO) presents exciting advancements~\cite{kovachki2023neural}, such as DeepONet \cite{lu2021learning}, and Fourier Neural Operator (FNO) \cite{li2020fourier}. These neural operators offer distinctive advantages compared to other traditional ML methods due to their simple settings in solving complex problems. It is of great values for solving PDEs and backstepping kernel equations. \cite{bhan2023operator} adopted neural operators for nonlinear adaptive control. The operator learning framework for accelerating nonlinear adaptive control was proposed. In addition, they apply the operator learning method for bypassing gain and control computations in PDE control \cite{bhan2023neural,qi2023neural,krstic2023neural}. It also can be used to accelerate PDE observer computations in traffic systems~\cite{shi2022machine}.

\textbf{Contributions}. 
In this paper, we design novel neural operator controllers for boundary stabilization of ARZ traffic PDE system. Two mappings, NO-approximated gain kernels $\mathcal{K}$ and NO-approximated control law $\mathcal{H}$, are developed to improve the computation speed of gain kernels and control law. The Lyapunov analysis is conducted to prove the practical local stability of  closed-loop system with the NO-approximated controllers. The simulations are also conducted to evaluate the performance of the two NO-approximated methods with several model-based controllers. It is shown that the NO-approximated methods accelerate computation speeds of gain kernels and control law and successfully stabilize the traffic system.

\section{Background for ARZ traffic control}
The macroscopic traffic on a given road can be described by a 2 $\times$ 2 nonlinear hyperbolic PDE using ARZ model \cite{aw2000resurrection, zhang2002non}. The model is defined by:
{\setlength\abovedisplayskip{0.1cm}
\setlength\belowdisplayskip{0.1cm}
\begin{align}
    \partial_t \rho + \partial_x (\rho v) &= 0, \label{origin1}\\
    \partial_t (v - V(\rho)) + v \partial_x(v - V(\rho)) &= \frac{V(\rho) - v}{\tau}, \label{origin2}\\
    \rho(0,t) &= \frac{q^\star}{v(0,t)}, \label{bc_q}\\
    v(L,t) &= v^\star + U(t), \label{bc_v}
\end{align}}
where $\rho(x,t)$ is traffic density, $v(x,t)$ denotes traffic velocity. $U(t)$ is the control input at the outlet of the road section which can be implemented by varying speed limit.
And the spatial and time domain are defined on $(x,t) \in [0, L]\times \mathbb{R}^+$. $\tau$ is the reaction time related to the drivers' behavior. $V(\rho)$ is the fundamental diagram which describes the relation between the traffic density and velocity. The fundamental diagram can be defined using, e.g., Greenshield's model:
{\setlength\abovedisplayskip{0.12cm}
\setlength\belowdisplayskip{0.12cm}
\begin{align}
    V(\rho) = v_f\left(1 - \left(\frac{\rho}{\rho_m}\right)^\gamma\right),
\end{align}}
where the $v_f$ is the maximum speed for the traffic flow, $\rho_m$ denotes the maximum density. And $(\rho^\star, v^\star)$  are the equilibrium points of the system. Also, using the fundamental diagram, we get $V(\rho^\star) = v^\star$, $q^\star = \rho^\star V(\rho^\star)$.
The linearized system is
    {\setlength\abovedisplayskip{0.1cm}
\setlength\belowdisplayskip{0cm}
    \begin{align}
    \partial_t \Tilde{w}(x,t) + \lambda_1 \partial_x \Tilde{w}(x,t) &= 0, \label{bs-q}\\
    \partial_t \Tilde{v}(x,t) - \lambda_2 \partial_x \Tilde{v}(x,t) &= c(x)\Tilde{w}(x,t),\\
    \Tilde{w}(0,t) &= - r \Tilde{v}(0,t),\\
    \Tilde{v}(L,t) &= U(t) \label{bs-bc_l},
    \end{align}}
 and the backstepping boundary control law is designed as
    {\setlength\abovedisplayskip{0cm}
\setlength\belowdisplayskip{0.1cm}
    \begin{align}\label{control_bs}
    U(t) = \int_0^L K^w(L,\xi)\Tilde{w}(\xi,t) d\xi + \int_0^L K^v(L,\xi)\Tilde{v}(\xi,t) d\xi.
    \end{align}}
    where $\lambda_1 = v^\star$, $\lambda_2 = - \rho^\star V'(\rho^\star) - v^\star$, $c(x) = - \frac{1}{\tau}\mathrm{e}^{-\frac{x}{\tau v^\star}}$, $r = \frac{-\rho^\star V'(\rho^\star) - v^\star}{v^\star}$.
    The corresponding backstepping transformation is  
    {\setlength\abovedisplayskip{0.1cm}
\setlength\belowdisplayskip{0.1cm}
    \begin{align}
    \alpha(x,t) &= \Tilde{w}(x,t),\label{back-1}\\
    \beta(x,t) &= \Tilde{v}(x,t) - \int_0^x K^w(x,\xi)\Tilde{w}(\xi,t)d\xi -  \int_0^x K^v(x,\xi)\Tilde{v}(\xi,t)d\xi, \label{back-2}
    \end{align}}
where the control gain kernels in \eqref{control_bs}, defined in the triangular domain $\mathcal{T} = \{ (x,\xi): 0 \leq \xi \leq x <L \}$, are computed using
    {\setlength\abovedisplayskip{0.1cm}
\setlength\belowdisplayskip{0.1cm}
    \begin{align}
    \lambda_2 K^w_x(x,\xi) - \lambda_1 K^w_\xi(x,\xi) &= c(x)K^v(x,\xi), \label{ker1}\\
    \lambda_2 K^v_x(x,\xi)  + \lambda_2 K^v_\xi(x,\xi) &= 0,\\
    K^w(x,x) &=-\frac{c(x)}{\lambda_1 + \lambda_2}, \\
    K^v(x,0) &= -K^w(x,0).\label{ker4}
    \end{align}}
We have the following theorem for the closed-loop system,
\begin{theorem}[Theorem 2, \cite{yu2019traffic}]
    The system \eqref{origin1} - \eqref{bc_v} with initial conditions $\rho(x,0)$, $v(x,0)$ $\in$ $L^2[0, L]$ and boundary controller \eqref{control_bs} is locally exponentially stable in $L_2$-sense at finite time $t_f = \frac{L}{\lambda_1} + \frac{L}{\lambda_2}$. 
\end{theorem}
\hspace{-0.5cm}
\section{Neural operator for approximating backstepping kernels}
%\subsection{Neural operator}
The neural operator is employed for approximating the function mapping. In the section, we introduce the neural operator using DeepONet for the mapping from the characteristic speed $\lambda_2$ to kernels $K^w(x,\xi)$, $K^v(x,\xi)$. An neural operator(NO) for approximating a nonlinear mapping $\mathcal{G}: \mathcal{U} \mapsto \mathcal{V}$
{\setlength\abovedisplayskip{0.1cm}
\setlength\belowdisplayskip{0.1cm}
\begin{align}
    \mathcal{G}_{\mathbb{N}}\left(\mathbf{u}_m\right)(y)=\sum_{k=1}^p g^{\mathcal{N}}\left(\mathbf{u}_m ; \vartheta^{(k)}\right) f^{\mathcal{N}}\left(y ; \theta^{(k)}\right),
\end{align}}where $\mathcal{U}, \mathcal{V}$ are function spaces of continuous functions $u\in \mathcal{U}$, $v\in\mathbf{V}$. $u_m$ is the evaluation of function $u$ at points $x_i = x_1,\dots,x_m$ . $p$ is the number of basis components in the target space, $y\in Y$ is the location of the output function $v(y)$ evaluations, and $g^{\mathcal{N}}$, $f^{\mathcal{N}}$ are NNs termed branch and trunk networks.
\begin{theorem}[DeepONet universal approximation theorem \cite{bhan2023neural}]\label{Deeptheo}
    Let $X \subset \mathbb{R}^{d_x}$, $Y$ $\subset$ $\mathbb{R}^{d_y}$ be compact sets of vectors $x\in X$ and $y\in Y$. Let $\mathcal{U}$: $X \rightarrow \mathcal{U} \subset \mathbb{R}^{d_u}$ and $\mathcal{V}$: $Y \rightarrow \mathcal{V} \subset \mathbb{R}^{d_v}$ be sets of continuous functions $u(x)$ and $v(y)$, respectively. Assume the operator $\mathcal{G}$: $\mathcal{U} \rightarrow \mathcal{V}$ is continuous. Then, for all $\epsilon > 0$, there exists a $m^*,p^* \in \mathbb{N}$ such that for each $m \geq m^*$, $p \geq p^*$, there exist $\theta^{(k)}$, $\vartheta^{(k)}$, neural networks $f^{\mathcal{N}}\left(\cdot ; \theta^{(k)}\right), g^{\mathcal{N}}\left(\cdot ; \vartheta^{(k)}\right), k= 1,\dots,p$ and $x_j \in X, j=1, \ldots, m$, with corresponding $\mathbf{u}_m=\left(u\left(x_1\right), u\left(x_2\right), \cdots, u\left(x_m\right)\right)^{\top}$, such that 
    {\setlength\abovedisplayskip{0.12cm}
\setlength\belowdisplayskip{0.12cm}
    \begin{align}
        \left|\mathcal{G}(u)(y)-\mathcal{G}_{\mathbb{N}}\left(\mathbf{u}_m\right)(y)\right|<\epsilon, 
    \end{align}}
    for all functions $u \in \mathcal{U}$ and all values $y\in Y$ of $\mathcal{G}(u) \in \mathcal{V}$.
\end{theorem}
\vspace{-0.3 cm}%%压缩图片后间隔
\begin{definition}
    The kernel operator $\mathcal{K}$: $\mathbb{R}^{+} \rightarrow C^1(\mathcal{T}) \times C^1(\mathcal{T})$ is defined by:
%     {\setlength\abovedisplayskip{0.1cm}
% \setlength\belowdisplayskip{0.1cm}
    \begin{align}
        K^{w}(x,\xi) &:= \mathcal{K}^w(\lambda_2)(x,\xi),\\
        K^{v}(x,\xi) &:= \mathcal{K}^v(\lambda_2)(x,\xi).
    \end{align}
\end{definition}
The neural operator $\mathcal{K}$ learns the mapping from the characteristic speed to backstepping transformation kernels. Based on Theorem \ref{Deeptheo}, we have the following lemma of the approximation of neural operator for the kernel equations:

\vspace{-0.3 cm}%%压缩图片后间隔
\begin{lemma}\label{NO-K}
    For all $\epsilon > 0$, there exists a neural operator $\mathcal{K}$ that for all $(x,\xi)\in \mathcal{T}$,
    \begin{align}\label{neuraloperator}
        \left| \mathcal{K}(\lambda_2)  - \mathcal{Q}(\lambda_2) \right| + \left|\partial_x\left( \mathcal{K}(\lambda_2) - \mathcal{Q}(\lambda_2) \right) \right| + \left|\partial_{\xi}\left( \mathcal{K}(\lambda_2)  - \mathcal{Q}(\lambda_2)\right) \right| < \epsilon.
    \end{align}
\end{lemma}
\begin{proof}
    The existence, uniqueness of the kernel equations have been proved in \cite{vazquez2011backstepping}. So the mapping $\mathcal{Q} : \mathbb{R}^+ \rightarrow C^1(\mathcal{T}) \times C^1(\mathcal{T})$ from $(\lambda_2)$ to $K^w(x,\xi),K^v(x,\xi)$ indicated by \eqref{ker1} - \eqref{ker4} and the solution of the kernel equations exists. The neural operator $\mathcal{K}$ approximates the backstepping kernels for a given $\lambda_2$ and their derivatives in the whole spatial-temporal domain. Using Theorem \ref{Deeptheo}, the sum of the absolute value of approximation errors is less than $\epsilon$.
\end{proof}
\vspace{-0.2cm}
%\subsection{Stability analysis for NO-approximated backstepping kernels}

We then provide the stability analysis of the ARZ traffic system with the NO-approximated kernels. We first start with the approximated kernels and put them into the ARZ system to get the NO-approximated target system. For a given value of $\lambda_2$, defining the output of the neural operator $\mathcal{K}(\lambda_2)(x,\xi)$:
{\setlength\abovedisplayskip{0cm}
\setlength\belowdisplayskip{0.1cm}
\begin{align}
    \Hat{K}^w &= \mathcal{K}^w(\lambda_2)(x,\xi),\\
    \Hat{K}^v &= \mathcal{K}^v(\lambda_2)(x,\xi).
\end{align}}
For the NO-approximated kernels $\mathcal{K}^w, \mathcal{K}^v$, we have the following result,
\begin{theorem}\label{es-NO-k}
    The PDE system \eqref{bs-q} - \eqref{bs-bc_l} is exponential stable under the control law \eqref{control_mu2k} with initial conditions $\Tilde{w}(x,0)$, $\Tilde{v}(x,0)$, satisfying
    {\setlength\abovedisplayskip{0.1cm}
\setlength\belowdisplayskip{0.1cm}
    \begin{align}
        ||(\Tilde{w},\Tilde{v})||^2_{L^2} \leq \mathrm{e}^{-\eta t}\frac{m_2}{m_1} ||(\Tilde{w}(x,0),\Tilde{v}(x,0))||^2_{L^2}, 
    \end{align}}
where $m_1>0, m_2>0, a>0, \eta = \nu - \frac{2a\epsilon(1+L)(2\lambda_2 + \lambda_1)}{m_1\lambda_2}$. The kernels are approximated by the neural operator \eqref{neuraloperator} with accuracy $\epsilon$. It also means that the system \eqref{origin1} - \eqref{bc_v} is locally exponential stable under the NO-approximated kernels and the system can eventually achieve to its equilibrium.
\end{theorem}
% \vspace{-0.9cm}
\begin{proof}
    We define the error for the NO-approximated kernels and backstepping kernels:
$\Tilde{K}^w(x,\xi) = K^w(x,\xi) - \Hat{K}^w(x,\xi)$, $\Tilde{K}^v(x,\xi) = K^v(x,\xi) - \Hat{K}^v(x,\xi)$,
and the baskstepping transformation then is turned into:
{\setlength\abovedisplayskip{0cm}
\setlength\belowdisplayskip{0.1cm}
\begin{align}
    \Hat{\alpha}(x,t) &= \Tilde{w}(x,t),\label{back-NO1}\\
    \Hat{\beta}(x,t) &= \Tilde{v}(x,t) - \int_0^x \Hat{K}^w(x,\xi)\Tilde{w}(\xi,t)d\xi -  \int_0^x \Hat{K}^v(x,\xi)\Tilde{v}(\xi,t)d\xi, \label{back-NO2}
\end{align}}
the corresponding backstepping control law is
{\setlength\abovedisplayskip{0.1cm}
\setlength\belowdisplayskip{0.1cm}
\begin{align}\label{control_mu2k}
    {U}(t) = \int_0^L \Hat{K}^w(L,\xi)\Tilde{w}(\xi,t) d\xi + \int_0^L \Hat{K}^v(L,\xi)\Tilde{v}(\xi,t) d\xi.
\end{align}}
Thus we get the target system with the NO-approximated kernels as 
{\setlength\abovedisplayskip{0.1cm}
\setlength\belowdisplayskip{0.1cm}
\begin{align}
    \partial_t \Hat{\alpha}(x,t) + \lambda_1 \partial_x\Hat{\alpha}(x,t) &= 0, \label{tar-NO1}\\
    \partial_t \Hat{\beta}(x,t)  - \lambda_2 \partial_x \Hat{\beta}(x,t) &=\lambda_2(\Tilde{K}^w(x,0) + \Tilde{K}^v(x,0))\Tilde{v}(0,t) + (\lambda_1 + \lambda_2) \Tilde{K}^w(x,x)\Tilde{w}(x,t)\nonumber\\
    & + \int_0^x (\lambda_2 \Tilde{K}^w_x(x,\xi) + \lambda_1 \Tilde{K}^w_{\xi}(x,\xi))\Tilde{w}(\xi,t) d\xi \nonumber\\
    &+\int_0^x (\lambda_2 \Tilde{K}^v_x(x,\xi) + \lambda_2  \Tilde{K}^v_{\xi}(x,\xi))\Tilde{v}(\xi,t) d\xi, \\
    \Hat{\alpha}(0,t) &= -r\Hat{\beta}(0,t),\\
    \Hat{\beta}(L,t) &= 0\label{tar-NO2}.
\end{align}}

For the target system \eqref{tar-NO1} - \eqref{tar-NO2} with the NO-approximated kernels, we define the Lyapunov candidate as
{\setlength\abovedisplayskip{0cm}
\setlength\belowdisplayskip{0.1cm}
\begin{align}\label{Lyap}
    V_k(t) = \int_0^L \frac{\mathrm{e}^{-\frac{\nu}{\lambda_1}x}}{\lambda_1}\Hat{\alpha}^2(x,t) + a\frac{\mathrm{e}^{-\frac{\nu}{\lambda_2}x}}{\lambda_2} \Hat{\beta}^2(x,t) dx,
\end{align}}
where the coefficients $\nu$ and $a$ are constants and $\nu>0$, $a>0$. The states of the NO-approximated backstepping traget system $(\Hat{\alpha},\Hat{\beta})$ and the original states $(\Tilde{w},\Tilde{v})$ have equivalent $L^2$ norms, the Lyapuov functional $V_k(t)$ is equivalent to the $L^2$ norm of the target system, so that there exist two constants $m_1>0$ and $m_2>0$, 
\begin{align}
    m_1||(\Tilde{w},\Tilde{v})||^2_{L^2} \leq V_k(t) \leq m_2||(\Tilde{w},\Tilde{v})||^2_{L^2}.
\end{align}
Taking time derivative along the trajectories of the system, and then we put into the system dynamics, integrating by parts, thus we get
% \begin{align}
%     \Dot{V}_k(t) = \int_0^L 2\frac{\mathrm{e}^{-\frac{\nu}{\lambda_1}x}}{\lambda_1}\Hat{\alpha}(x,t)\partial_t \Hat{\alpha}(x,t) + 2a\frac{\mathrm{e}^{-\frac{\nu}{\lambda_2}x}}{\lambda_2} \Hat{\beta}(x,t)\partial_t \Hat{\beta}(x,t) dx,
% \end{align}
\begin{align}
     \Dot{V}_k(t) &= -\nu V_k(t) + (r^2-a) \Hat{\beta}^2(0,t) -\mathrm{e}^{-\frac{\nu}{\lambda_1}L}\Hat{\alpha}^2(L,t) \nonumber\\
     &+ \int_0^L 2a\frac{\mathrm{e}^{-\frac{\nu}{\lambda_2}x}}{\lambda_2}\Hat{\beta}(x,t) \left( \lambda_2 (\Tilde{K}^w(x,0) + \Tilde{K}^v(x,0))\Tilde{v}(0,t)+(\lambda_1+\lambda_2) \Tilde{K}^w(x,x)\Tilde{w}(x,t)\right.\nonumber\\
    & + \int_0^x (\lambda_2 \Tilde{K}^w_x(x,\xi) + \lambda_1 \Tilde{K}^w_{\xi}(x,\xi))\Tilde{w}(\xi,t) d\xi \nonumber\\
    &+\int_0^x (\lambda_2 \Tilde{K}^v_x(x,\xi) + \lambda_2  \Tilde{K}^v_{\xi}(x,\xi))\Tilde{v}(\xi,t) d\xi \left.\right)dx,
\end{align}
For the integral term, we take the norm and using the Young inequality and Caushy-Schwarz inequality, and then combining the equivalent norm of the Lyapunov candiate, we have:
{\setlength\abovedisplayskip{0.12cm}
\setlength\belowdisplayskip{0.12cm}
\begin{align}
    \int_0^L \norm{2a\frac{\mathrm{e}^{-\frac{\nu}{\lambda_2}x}}{\lambda_2}\Hat{\beta}(x,t) \lambda_2 (\Tilde{K}^w(x,0) + \Tilde{K}^v(x,0))\Tilde{v}(0,t)}dx
    \leq \frac{2a\epsilon}{m_1}V_k(t) + 2aL\epsilon \Hat{\beta}^2(0,t).
\end{align}}
Using the same method, we can easily get the results for the other terms of the Lyapunov candidate,
\begin{align}
    &\int_0^L \norm{2a\frac{\mathrm{e}^{-\frac{\nu}{\lambda_2}x}}{\lambda_2}\Hat{\beta}(x,t)(\lambda_1+\lambda_2) \Tilde{K}^w(x,x)\Tilde{w}(x,t)}dx \leq \frac{2a\epsilon(\lambda_1+\lambda_2)}{m_1\lambda_2}V_k(t),\\
    &\int_0^L \norm{2a\frac{\mathrm{e}^{-\frac{\nu}{\lambda_2}x}}{\lambda_2}\Hat{\beta}(x,t)+ \int_0^x (\lambda_2 \Tilde{K}^w_x(x,\xi) + \lambda_1 \Tilde{K}^w_{\xi}(x,\xi))\Tilde{w}(\xi,t) d\xi}dx \leq \frac{2a\epsilon(\lambda_1+\lambda_2)L}{m_1 \lambda_2}V_k(t),\\
    & \int_0^L \norm{2a\frac{\mathrm{e}^{-\frac{\nu}{\lambda_2}x}}{\lambda_2}\Hat{\beta}(x,t)+ \int_0^x (\lambda_2 \Tilde{K}^v_x(x,\xi) + \lambda_2  \Tilde{K}^v_{\xi}(x,\xi))\Tilde{v}(\xi,t) d\xi}dx \leq \frac{2a\epsilon L}{m_1}V_k(t),
\end{align}
thus we get the following Lyapunov candidate,
{\setlength\abovedisplayskip{0.1cm}
\setlength\belowdisplayskip{0.1cm}
\begin{align}
    \Dot{V}_k(t) \leq - \eta V_k(t) + \left( r^2-a+2aL\epsilon \right)\Hat{\beta}^2(0,t) - \mathrm{e}^{-\frac{\nu}{\lambda_1}L}\Hat{\alpha}^2(L,t),
\end{align}}
where $\eta = \nu - \frac{2a\epsilon(1+L)(2\lambda_2 + \lambda_1)}{m_1\lambda_2}$. The coefficients $\nu, \epsilon, a$ are chosen such that
{\setlength\abovedisplayskip{0.1cm}
\setlength\belowdisplayskip{0.1cm}
\begin{align}
    \eta >0, r^2-a+2aL\epsilon < 0.
\end{align}}
So we get the following result:
{\setlength\abovedisplayskip{0.1cm}
\setlength\belowdisplayskip{0.1cm}
\begin{align}
    \Dot{V}_k(t) &\leq -\eta V_k(t) \rightarrow V_k(t) \leq V(0)\mathrm{e}^{-\eta t}
\end{align}}
Using the equivalent norm of the Lyapunpov functional, we have:
{\setlength\abovedisplayskip{0.1cm}
\setlength\belowdisplayskip{0.1cm}
\begin{align}
    ||(\Tilde{w},\Tilde{v})||^2_{L^2} \leq \mathrm{e}^{-\eta t}\frac{m_2}{m_1} ||(\Tilde{w}(x,0),\Tilde{v}(x,0))||^2_{L^2}.
\end{align}}
The exponential stability thus is proved.
\end{proof}
\vspace{-0.5cm}
\section{NO-approximated backstepping control law}
Based on the previous section, we have proved that the PDE system is exponentially stable with the NO-approximated kernels. In this section, we consider whether the neural operator can directly approximate the mapping from the characteristic speed to control law \eqref{control_bs} rather than approximating the backstepping kernels. The form of stability we get in this section is also only achieved in the practical sense using the trained neural operator. Different with the results in \cite{bhan2023neural}, we train the mapping only from the characteristic speed $\lambda_2$, while in \cite{bhan2023neural} they trained the mapping from the parameter $\beta(x)$ and the system state $u(x,t)$ at its corresponding time step to the backstepping control law. It costs more training time due to the number of model parameters.

Recalling the control law \eqref{control_bs}, we define the operator mapping $\mathcal{H}\left(\lambda_2\right): \mathbb{R}^+ \rightarrow \mathbb{R} $ from  $\lambda_2$ to $U(t)$. From the expression of backstepping control law \eqref{control_bs}, it is shown that there is no explicit form for the mapping from $\lambda_2$ to $U(t)$. The relation between the $\lambda_2$ and $U(t)$ is characterized with the kernel equations \eqref{ker1} - \eqref{ker4}. The control law for mapping is
{\setlength\abovedisplayskip{0.1cm}
\setlength\belowdisplayskip{0.1cm}
\begin{align}\label{controlmu2c}
    U(t) = \mathcal{H}\left(\lambda_2\right)(L, t),
\end{align}}and the NO-approximated mapping for $\mathcal{H}\left(\lambda_2\right): \mathbb{R}^+ \rightarrow \mathbb{R} $ is defined as $\mathcal{\Hat{H}}(\lambda_2): \mathbb{R}^+ \rightarrow \mathbb{R}$. It can be found that the NO-approximated mapping $\mathcal{\Hat{H}}(\lambda_2)$ has the following lemma based on Theorem \ref{Deeptheo},
\begin{lemma}
    For $\epsilon > 0$,  there exists a neural operator $\mathcal{\Hat{H}}(\lambda_2)$ that can approximate the control law mapping in the spatial-temporal domain $(x,t) \in [0,L]\times \mathbb{R}^+$:
    {\setlength\abovedisplayskip{0.1cm}
\setlength\belowdisplayskip{0.1cm}
    \begin{align}
        \left| \mathcal{H}\left(\lambda_2\right)(L) - \mathcal{\Hat{H}}(\lambda_2)(L) \right| < \epsilon.
    \end{align}}
\end{lemma}
Applying the NO-approximated control law to the system, we get the target system as:
{\setlength\abovedisplayskip{0.1cm}
\setlength\belowdisplayskip{0.1cm}
\begin{align}
    \partial_t \Check{\alpha}(x,t) + \lambda_1 \partial_x\Check{\alpha}(x,t) &= 0, \label{nocontrol-q}\\
    \partial_t \Check{\beta}(x,t) - \lambda_2 \partial_x \Check{\beta}(x,t) &= 0,\\
    \Check{\alpha}(0,t) &= -r\Check{\beta}(0,t),\\
    \Check{\beta}(L,t) &= \mathcal{H}\left(\lambda_2\right)(L, t)  - \mathcal{\Hat{H}}(\lambda_2)(L, t). \label{nocontrol-bc_l}
\end{align}}
Compared with the target system in \cite{yu2019traffic}, the system \eqref{nocontrol-q} - \eqref{nocontrol-bc_l} is not strictly exponential stable due to the approximated error of NO-control law. We have the following theorem for the NO-approximated control law,
\begin{theorem}
    The system \eqref{bs-q} - \eqref{bs-bc_l} is locally practically exponentially stable under the control law \eqref{controlmu2c} with initial conditions $\Tilde{w}(x,0),\Tilde{v}(x,0)$, the kernels are the same as \eqref{ker1} - \eqref{ker4}, such that 
    {\setlength\abovedisplayskip{0.1cm}
\setlength\belowdisplayskip{0.1cm}
    \begin{align}
        ||(\Tilde{w},\Tilde{v})||^2_{L^2} \leq \frac{k_2}{k_1}\mathrm{e}^{-\nu t} ||(\Tilde{w}(x,0),\Tilde{v}(x,0))||^2_{L^2} + \frac{a}{k_1} \mathrm{e}^{-\frac{\nu}{\lambda_2}L}{\epsilon^2}.
    \end{align}}
    
\end{theorem}
\vspace{-0.15 cm}%%压缩图片后间隔
\begin{proof}
    For the NO-approximated control law, using the Lyapunov candidate again to analyze the stability of the target system \eqref{nocontrol-q} - \eqref{nocontrol-bc_l}. 
    {\setlength\abovedisplayskip{0.1cm}
\setlength\belowdisplayskip{0.1cm}
\begin{align}
    V_U(t) = \int_0^L \frac{\mathrm{e}^{-\frac{\nu}{\lambda_1}x}}{\lambda_1}\Check{\alpha}^2(x,t) + a\frac{\mathrm{e}^{-\frac{\nu}{\lambda_2}x}}{\lambda_2} \Check{\beta}^2(x,t) dx,
\end{align}}
where the coefficients $\nu$ and $a$ are the same with before. The Lyapunov functional also has the following equivalent norm with $k_1>0,k_2 >0$
{\setlength\abovedisplayskip{0.1cm}
\setlength\belowdisplayskip{0.1cm}
\begin{align}\label{eqnormVu}
    k_1||(\Tilde{w},\Tilde{v})||^2_{L^2} \leq V_U(t) \leq k_2||(\Tilde{w},\Tilde{v})||^2_{L^2}.
\end{align}}
Taking time derivative along the trajectories, putting into the system dynamics, integrating by parts, we have:
{\setlength\abovedisplayskip{0cm}
\setlength\belowdisplayskip{0.1cm}
\begin{align}
    \Dot{V}_U(t) = -\nu {V}_U(t) + (r^2 - a)\Check{\beta}^2(0,t) + a\mathrm{e}^{-\frac{\nu}{\lambda_2}L}\Check{\beta}^2(L,t) -\mathrm{e}^{-\frac{\nu}{\lambda_1}L}\Check{\alpha}^2(L,t).
\end{align}}
The term $a\mathrm{e}^{-\frac{\nu}{\lambda_2}L}\Check{\beta}^2(L,t)$ is equal to 0 in the ideal situation when the NO-approximated mapping achieves $100\%$ accuracy of approximation which means that $\mathcal{H}\left(\lambda_2\right)(L) - \mathcal{\Hat{H}}(\lambda_2)(L) = 0$ and we can easily get the exponential stability for the system. Here the mapping has the error $\epsilon$. We take the $r^2 -a \leq 0$, and we get
{\setlength\abovedisplayskip{0.1cm}
\setlength\belowdisplayskip{0.1cm}
\begin{align}
   V_U(t) \leq V_U(0) \mathrm{e}^{-\nu t} + a \mathrm{e}^{-\frac{\nu}{\lambda_2}L}\sup_{0 \leq \varsigma \leq t} (\mathcal{H}\left(\lambda_2\right)(L) - \mathcal{\Hat{H}}(\lambda_2)(L))^2(L,\varsigma).
\end{align}}
Using \eqref{eqnormVu}, we have
{\setlength\abovedisplayskip{0.1cm}
\setlength\belowdisplayskip{0.1cm}
\begin{align}
    ||(\Tilde{w},\Tilde{v})||^2_{L^2} \leq \frac{k_2}{k_1}\mathrm{e}^{-\nu t} ||(\Tilde{w}(x,0),\Tilde{v}(x,0))||^2_{L^2} + \frac{a}{k_1} \mathrm{e}^{-\frac{\nu}{\lambda_2}L}\epsilon^2.
\end{align}}
Thus we have proved that the system \eqref{bs-q} - \eqref{bs-bc_l} is locally practically exponentially stable.
\end{proof}

\section{Experiments}
In this section, we present and analyze the performance of the proposed neural operator controllers for the ARZ traffic PDE system, and also provide two comparisons with model-based controllers: (i) backstepping controller (ii) PI controller.

We first test the efficacy of the controller with \textit{NO-approximated backstepping kernels}.
To train the neural operator, we use numerical simulations to get data for training. We run the simulation on a $L=500 \text{m}$ long road and the simulation time is $T=300\text{s}$. The free-flow velocity is $v_m = 40 \text{m/s}$, the maximum density is $\rho_m = 160 \text{veh/km}$, the equilibrium density is selected as $\rho^\star = 120 \text{veh/km}$, the reaction time for the drives adapting to speed is $\tau = 60 \text{s}$, and $\gamma = 1$. For the initial conditions, we take the sinusoidal inputs as $\rho(x,0) = \rho^\star + 0.1 \sin \left( \frac{3\pi x}{L}\right)\rho^\star$, $v(x,0) = v^\star - 0.1\sin \left( \frac{3\pi x}{L}\right)v^\star$ to mimic the stop-and-go traffic wave.
To get enough data for training, we use 900 different values for $\rho^\star \in [90 \text{veh/km},130\text{veh/km}]$ to get the different value for $\lambda_2 \in [5,25]$ and different kernels $K^w(x,\xi), K^v(x,\xi)$. And we train the model on an Nvidia RTX 4090Ti GPU.  Using the trained neural operator model, we run the simulation with the same parameters. The results for backstepping controller and NO-approximated controller is shown in Fig. \ref{bs-res} and \ref{mu2k-res}. The blue line denotes the initial condition while the red line represents the boundary condition of the system. It can be found that the neural operator can still stabilize the traffic system compared with the backstepping control method. The traffic density and velocity all converge to their equilibrium point $\rho^\star = 120 \text{veh/km}$, $v^\star = 36 \text{m/s}$. The error between the closed-loop result of the backstepping controller and the approximated one is shown in Fig. \ref{mu2k-err}. 
\begin{figure}[t]
\centering
\subfigure[Density]{\includegraphics[width=0.45\linewidth]{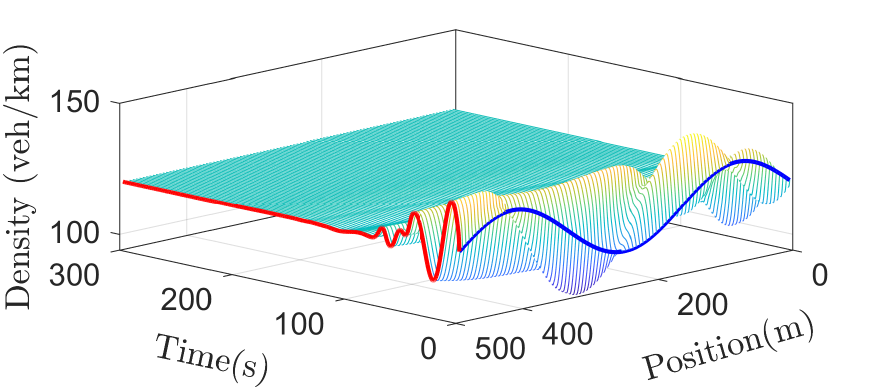}}
\subfigure[Velocity]{\includegraphics[width=0.45\linewidth]{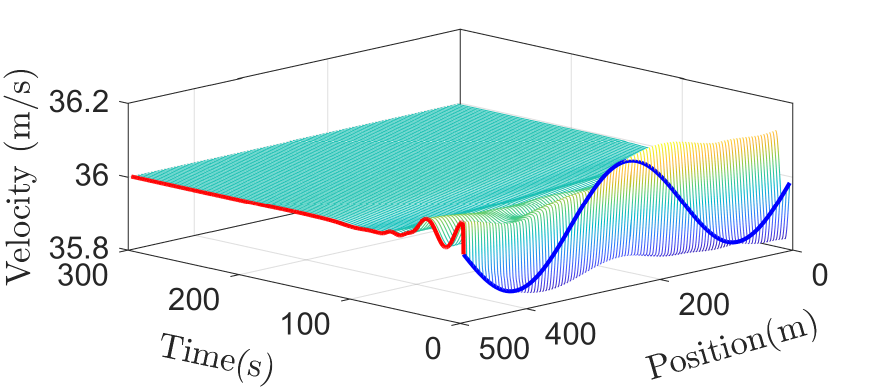}}
%\vspace{-0.3 cm}%%压缩图片后间隔
\caption{Backstepping controller for closed-loop system}
\label{bs-res}
%\vspace{-0.3 cm}%%压缩图片后间隔
\end{figure}
%\vspace{-0.5 cm}%%压缩图片后间隔
\begin{figure}[t]
\centering
\subfigure[Density]{\includegraphics[width=0.45\linewidth]{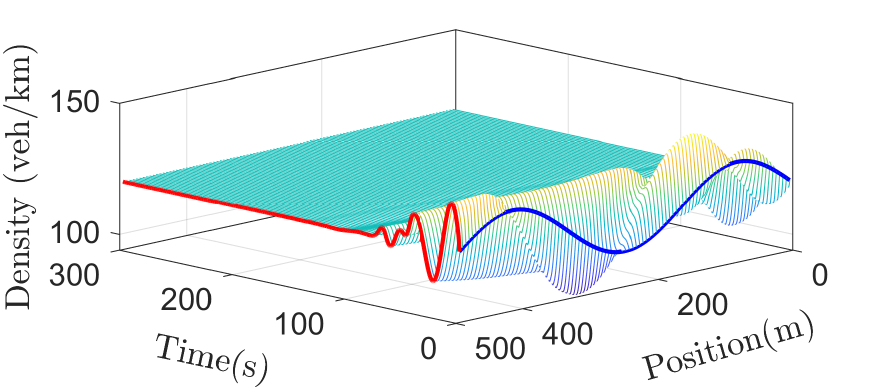}}
\subfigure[Velocity]{\includegraphics[width=0.45\linewidth]{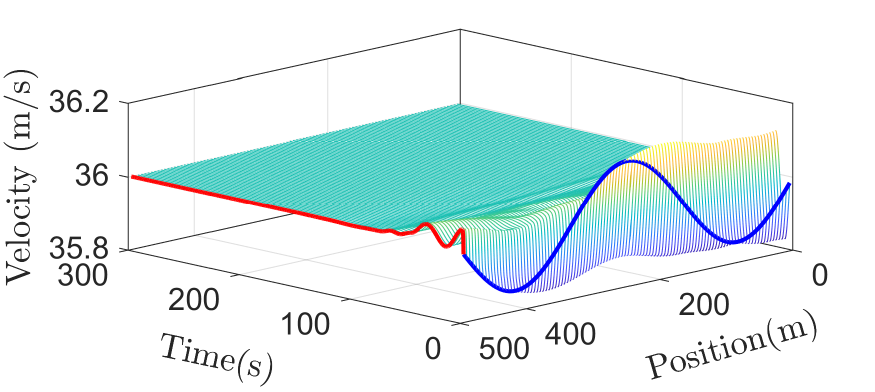}}
%%压缩图片后间隔
\caption{Neural operator for $\lambda_2 \rightarrow  K^w(x,\xi), K^v(x,\xi)$}
\label{mu2k-res}
% \vspace{-0.5 cm}%%压缩图片后间隔
\end{figure}
% \vspace{-0.8 cm}
\begin{figure}[t]
\centering
\subfigure[Density]{\includegraphics[width=0.45\linewidth]{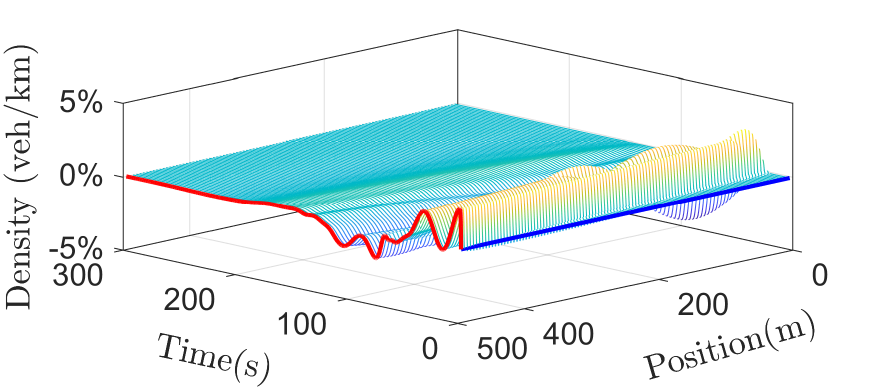}}\hfil
\subfigure[Velocity]{\includegraphics[width=0.45\linewidth]{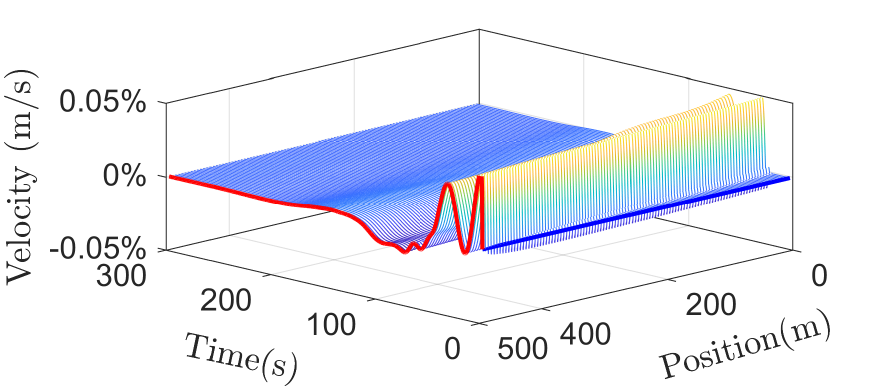}}
%%压缩图片后间隔
\caption{Error between backstepping control and neural operator for kernels}
\label{mu2k-err}
% \vspace{-0.5 cm}%%压缩图片后间隔
\end{figure}

% ------------------- ------------------%

We then test the \textit{direct NO-approximated controller}. Using the same parameter settings in previous section to get the training data including 900 instance. And we get the result for approximated control law mapping is shown in Fig. \ref{mu2con-res}. We also set the backstepping controller as the baseline to evaluate the performance of the neural operator. The error between the approximated control law and backstepping controller is shown in Fig. \ref{err-mu2control}. It is shown that the density and velocity do not converge to their equilibrium point. The density and velocity error exist in the whole simulation time. This is reasonable because the neural operator only learns the mapping from $\lambda_2$ to $U(t)$. As we know, the control law in \eqref{control_bs} are composed with two parts, one is the backstepping kernels another is the system states at the current time step.
% \begin{figure}
% \centering
% \includegraphics[width = 0.45\linewidth]{image/Train_plot_mu2control.pdf}
% \caption{Train loss and test loss}
% \label{loss-mu2control}
% \end{figure}

% \begin{figure}
% \centering
% \subfigure[NO-Kernel]{\includegraphics[width=0.45\linewidth]{image/Train_plot_mu2k.pdf}}\hfil
% \subfigure[NO-Control law]{\includegraphics[width=0.45\linewidth]{image/Train_plot_mu2control.pdf}}
% \caption{Train loss and test loss}
% \label{mu2con-res}
% \end{figure}

\begin{figure}[tbp]
\centering
\subfigure[Density]{\includegraphics[width=0.45\linewidth]{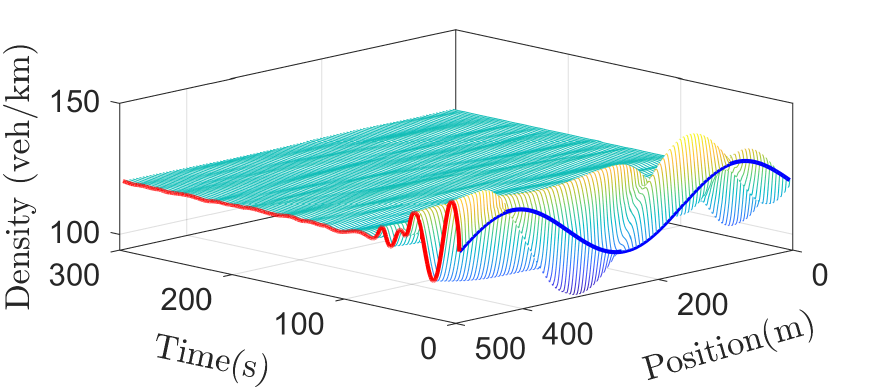}}
\subfigure[Velocity]{\includegraphics[width=0.45\linewidth]{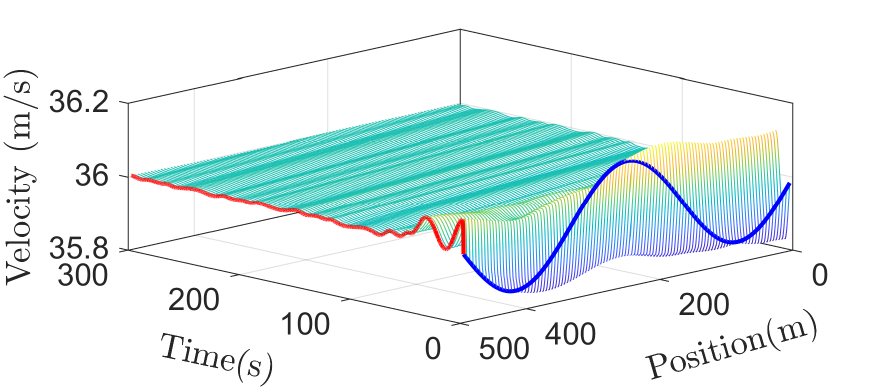}}
%\vspace{-0.3 cm}%%压缩图片后间隔
\caption{Neural operator for $\lambda_2 \rightarrow U(t)$}
%\vspace{-0.3 cm}%%压缩图片后间隔
\label{mu2con-res}
\end{figure}
%\vspace{-0.3 cm}%%压缩图片后间隔
\begin{figure}[t]
\centering
\subfigure[Density]{\includegraphics[width=0.45\linewidth]{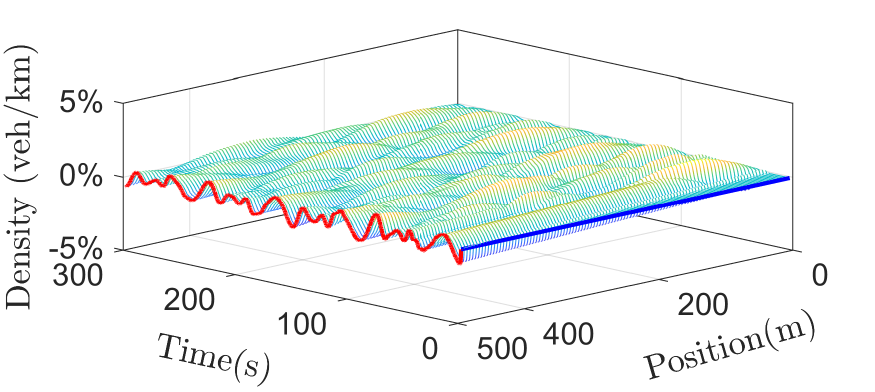}}\hfil
\subfigure[Velocity]{\includegraphics[width=0.45\linewidth]{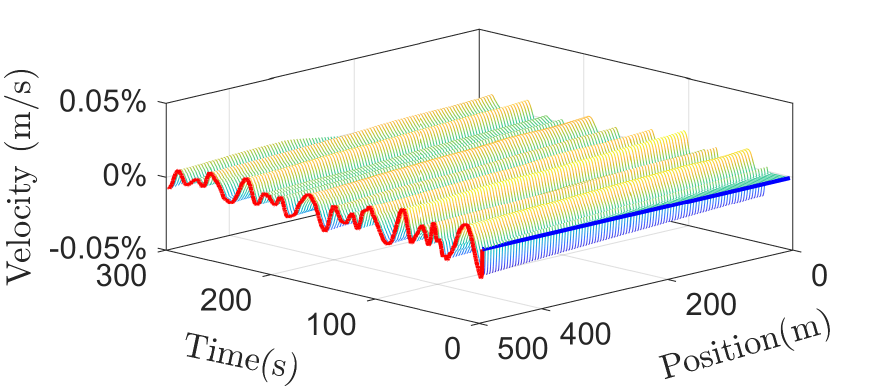}}
%\vspace{-0.3 cm}%%压缩图片后间隔
\caption{Error between backstepping control and neural operator for control law}
\label{err-mu2control}
% \vspace{-0.5 cm}%%压缩图片后间隔
\end{figure}

%\subsection{Compared with PI controller}
We also provide a comparison with the PI controller. It was shown that the PI boundary control can also stabilize the traffic system \cite{zhang2019pi}. We also apply the PI controller at the outlet of the road section as $\Tilde{v}(L,t) = U_{PI}(t)$. The control law is given as:
{\setlength\abovedisplayskip{0cm}
\setlength\belowdisplayskip{0.12cm}
\begin{align}
    U_{PI}(t) = v^\star + k_p^v (v(0,t) - v^\star) + k_i^v \int_0^t (v(0,t) - v^\star) ds,
\end{align}}where $k_p^v$, $k_i^v$ are tuning gains. The results for PI control are shown in Fig. \ref{PI-res}.
\begin{figure}[t]
\centering
\subfigure[Density]{\includegraphics[width=0.45\linewidth]{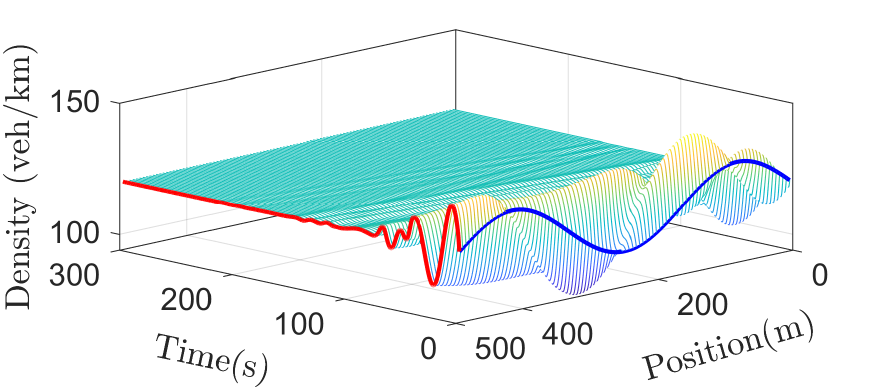}}\hfil
\subfigure[Velocity]{\includegraphics[width=0.45\linewidth]{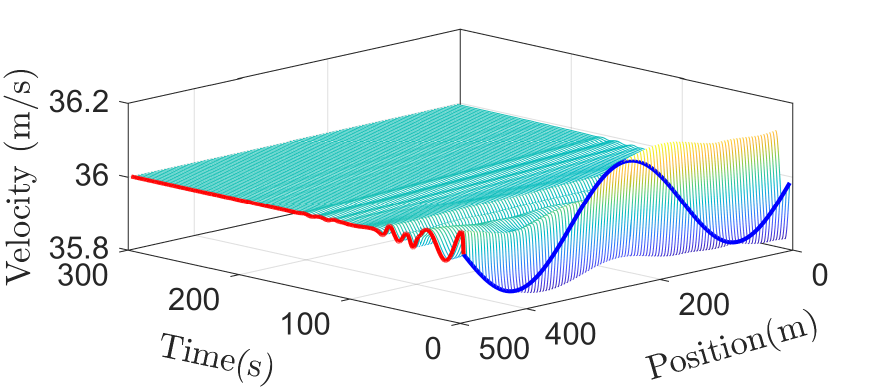}}
\vspace{-0.3 cm}%%压缩图片后间隔
\caption{PI controller}
\vspace{-0.4 cm}%%压缩图片后间隔
\label{PI-res}
\end{figure}
\begin{figure}[h!]
\centering
\subfigure[Control input $U(t)$]{\includegraphics[width=0.45\linewidth]{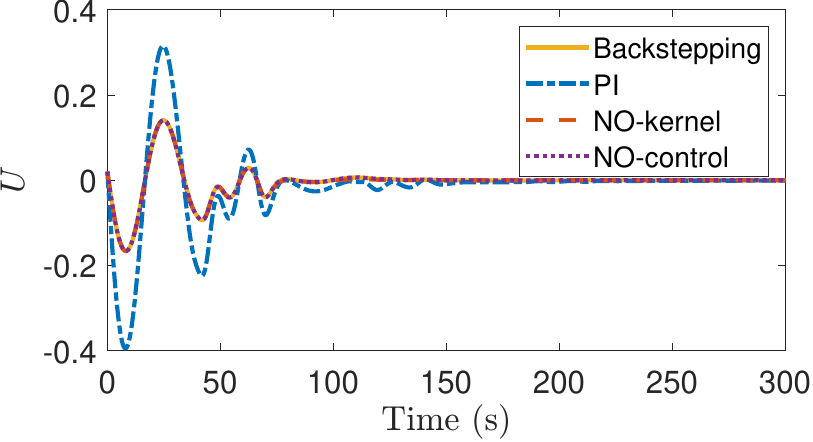}}\hfil
\subfigure[Norm of the states]{\includegraphics[width=0.43\linewidth]{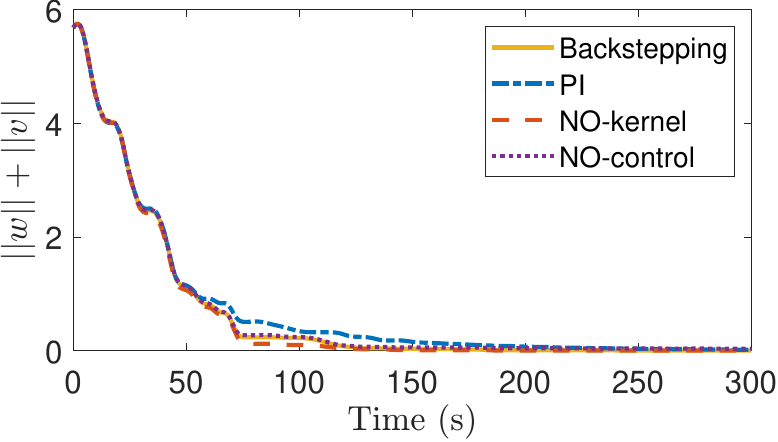}}
\vspace{-0.3 cm}%%压缩图片后间隔
\caption{Comparison of $U(t)$ and states norm}
\label{compare-U}
\vspace{-0.3  cm}%%压缩图片后间隔
\end{figure}
For the control law mapping, we give the comparisons of control law and norm of states, which are shown in Fig. \ref{compare-U}. From the results of $U(t)$,  the neural operator can approximate the backstepping control law while the PI controller needs more control effort to stabilize the traffic system. All the four controllers are eventually stabilize the system. However, the norm of the states of PI controller converges to zero slower than other three controllers. It is shown that the NO-approximated kernels and directly approxmated control law achieves satisfying closed-loop results.

The computation time of neural operator, backstepping controller, PI controller  are shown in Tab. \ref{Tab-res-1}. We set the bacsktepping control method as the baseline of the system. It is can be found that the average computation time of NO-approximated methods are faster than backsteping controller and PI controller which gives the possibility of accelerating the online application in real traffic system. The NO-approximated methods are also has lower relative error compared with PI controller with a faster computation speed. 
% , \ref{Tab-res-2}
\vspace{-0.3 cm}%%压缩图片后间隔
\begin{table}[tbp!]
    \centering
    \begin{tabular}{|c|c|c|}
    \hline
       \textbf{Method}  & \makecell[c]{\textbf{Average} \\ \textbf{Computation Time}} & \makecell[c]{\textbf{Average} \\ \textbf{$L_2$ Error}} \\
       \hline
       Backstepping controller  & 0.0847s & 0 \\
       \hline
       NO-approximated kernels &  0.0021s & 0.0317 \\
       \hline
        NO-approximated control law & 0.0012s & 0.0266\\
       \hline
       PI controller &  0.0033s & 0.0936 \\
       \hline
    \end{tabular}
    \vspace{-0.2 cm}%%压缩图片后间隔
    \caption{The computation time and average $L_2$ error of different controllers and neural operators}
    \label{Tab-res-1}

\end{table}

% \begin{table}[tbp]
%     \centering
%     \begin{tabular}{|c|c|c|c|}
%     \hline
%        \textbf{Method}  & \makecell[c]{\textbf{Average} \\ \textbf{Calculation Time}} & \makecell[c]{\textbf{Average} \\ \textbf{Relative Error}}& \textbf{Training Time} \\
%        \hline
%        Backstepping control scheme  & 0 & 0 & none\\
%        \hline
%        NO-approximated control law & 0.0012 & 0.1455 & 0\\
%        \hline
%        PI controller &  0.0033 & 0.5126 & none\\
%        \hline
%     \end{tabular}
%     \caption{The results for No-approximated control law}
%     \label{Tab-res-2}
% \end{table}
% The result for $L_2$ error of PDE states is shown in Fig. \ref{err-l2}
% \begin{figure}
% \centering
% \subfigure[norm2mu2control]{\includegraphics[width=0.45\linewidth]{image/norm2_mu2control.pdf}}\hfil
% \subfigure[norm2mu2k]{\includegraphics[width=0.45\linewidth]{image/norm2_mu2k.pdf}}
% \caption{norm2 backstepping}
% \label{err-l2}
% \end{figure}

% \begin{figure}
% \centering
% \subfigure[Density]{\includegraphics[width=0.45\linewidth]{image/density_err.pdf}}\hfil
% \subfigure[Velocity]{\includegraphics[width=0.45\linewidth]{image/velocity_err.pdf}}
%  \caption{Error between numerical and control law mapping}
% \end{figure}

\section{Conclusion}
In this paper, we proposed an operator learning framework for boundary control of traffic systems. The ARZ PDE model is adopted to describe the spatial-temporal evolution of traffic density and velocity. We define first the operator mapping from the model parameter, i.e., characteristic speed to backstepping control kernels  and then the mapping directly to a boundary control law . The neural operators using DeepONet are trained to approximate the two mappings. We also derived the theoretical stability for the NO-approximated closed-loop system using Lyapunov analysis. The simulation results showed that the NO-approximated mappings have satisfying accuracy and significantly accelerate the computation process. One of the future work is to incorporate real traffic data into training of neural operator.

% Acknowledgments---Will not appear in anonymized version
% \acks{We thank a bunch of people.}

\bibliography{reference}

\end{document}